%% file: main.tex
\documentclass{article}
\usepackage[final]{corl_2020}
\usepackage[utf8]{inputenc}
\usepackage{amsmath,amssymb}
\usepackage{mathtools}
\usepackage{amsthm}
\usepackage{graphicx}
\usepackage{float}
\usepackage{wrapfig}
\usepackage{subcaption}
\usepackage{subfloat}
\usepackage[ruled]{algorithm2e}
\usepackage{amsfonts,dsfont}
\theoremstyle{plain}
\newtheorem{thm}{\textbf{Theorem}}
\newtheorem{lem}{\textbf{Lemma}}

\DeclarePairedDelimiter\set{\{}{\}}

\theoremstyle{remark}
\newtheorem{rem}{\textbf{Remark}}

\newcommand{\I}{\mathcal{I}}
\newcommand{\R}{\mathcal{R}}
\newcommand{\prob}{\mathbb{P}}

\newcommand{\distribution}{\mathcal{D}}
\newcommand{\regpara}{\eta}

\newcommand{\reals}{\mathbb{R}}
\title{Reactive motion planning with probabilistic \\ safety guarantees}
\author{Yuxiao Chen, Ugo Rosolia, Chuchu Fan, Aaron D. Ames, and Richard Murray\\
 California Institute of Technology\\
  \texttt{\{chenyx,urosolia,chuchu,ames,murray.cds\}@caltech.edu} \\}

\date{March 2019}

\begin{document}
\maketitle

\begin{abstract}
Motion planning in environments with multiple agents is critical to many important autonomous applications such as autonomous vehicles and assistive robots. This paper considers the problem of motion planning, where the controlled agent shares the environment with multiple uncontrolled agents. First, a predictive model of the uncontrolled agents is trained to predict all possible trajectories within a short horizon based on the scenario. The prediction is then fed to a motion planning module based on model predictive control. We proved generalization bound for the predictive model using three different methods, post-bloating, support vector machine (SVM), and conformal analysis, all capable of generating stochastic guarantees of the correctness of the predictor. The proposed approach is demonstrated in simulation in a scenario emulating autonomous highway driving.
\end{abstract}
\section{Introduction}

Many of the important applications of autonomy contain multiple agents, many of which are not under the control of the autonomous system, and we refer to them as the uncontrolled agents. In order to achieve safe operation in a multi-agent environment, the behavior of uncontrolled agents needs to be modeled and taken advantage of. Take autonomous driving as an example. When planning the path for the autonomous vehicle, the behavior of other road users such as human-driven vehicles and pedestrians is critical to the safety of the autonomous vehicle. Modelling of uncontrolled agents has been an important research problem and received a lot of attention. The simplest setting is perhaps assuming a fixed trajectory of the uncontrolled agents \cite{kuwata2008motion}, which is obviously too optimistic and may cause collision. The other extreme is to over-approximate the reachable set of the uncontrolled agents \cite{chen2017fast}, which is typically too conservative. It has been noticed that a predictive model is needed to plan safe yet not an overly conservative motion for the autonomous vehicle. In particular, for safety-critical applications such as autonomous driving, one would want guarantees of the correctness of such models, which lead to the guarantees of successful task fulfillment. Typically, predictive models are learned from observations of the uncontrolled agents, using various representation structures such as temporal logic formulae \cite{jha2017telex,vazquez2018learning,JinX}, Gaussian process \cite{trautman2010unfreezing,aoude2013probabilistically}, Inverse Reinforcement Learning (IRL) \cite{ziebart2010modeling,sadigh2016planning,Kuderer2015}, and Generative Adversarial Network (GAN) \cite{bhattacharyya2018multi}. 

For agents that exhibit complex behavior, accurately modeling their behavior is almost impossible due to limitations such as the expressiveness of the model, the amount of data required, and above all, the nondeterministic nature of such agents. For instance, a human is a typical type of nondeterministic agent whose behavior under the same scenario is usually inconsistent. One way to deal with this inconsistency is to fit a probabilistic model. In particular, Markovian models such as Markov chains \cite{althoff2009model}, hidden Markov models \cite{kumagai2006prediction}, and partially observed Markov decision processes \cite{bai2015intention} are a popular choice since they simplify the reasoning with the Markovian property. Another class of approaches is to directly model the nondeterministic behavior with set-based methods, including the GAN-based prediction \cite{bhattacharyya2018multi}, SVM approach \cite{chen2018modelling,chen2019counter}, and the Covernet \cite{phan2019covernet}, which uses a neural network. If such a predictive model has a high probability of including all possible behavior of the uncontrolled agent, safety can be guaranteed as long as the motion planner plans a trajectory that is not in collision with the predicted trajectories. For this reason, this class of approaches is more amenable to safety-critical applications.

One important issue to consider is the compatibility of the predictive model with motion planning methods. As a counterexample, predictive models for instantaneous actions such as \cite{chen2019counter,bhattacharyya2018multi} are not compatible with most of the popular motion planning tools such as rapid random tree (RRT) \cite{lavalle2001rapidly} and model predictive control (MPC)~\cite{borrelli2017predictive} since these motion planning tools typically consider a finite horizon, i.e., they conduct sequential decision making. If one were to use such a predictive model in a sequential decision-making framework, according to the predictive model, the set of possible actions for future steps depends on the scenario in the future, which is a function of the actions of both the controlled agent and the uncontrolled agents, thus nondeterministic. Therefore, the whole horizon planning problem becomes convoluted and nondeterministic. In addition, the representation of the predictive model typically consists of complicated functions for the sake of expressibility, which poses additional difficulty to the horizon planning problem. There exist methods that provide prediction with a finite horizon, such as clustering and classification with a finite set of trajectories \cite{vasquez2004high,shah2016resolution} and the Covernet \cite{phan2019covernet}. However, to the best of our knowledge, no existing predictive model for motion planning provides guarantees on the correctness of the prediction.

In this paper, we propose a classification approach that learns a predictive model capable of predicting possible trajectories within a short horizon, and we are able to prove a probabilistic guarantee on the correctness of the prediction. To do so, the output of the predictive model is a set of possible trajectories with a fixed horizon, which can be directly utilized by the motion planner to plan safe trajectories. Different from the existing research on data-driven verification of autonomous systems \cite{Balkan,DRYVR,sofiedatadriven}, this paper focuses on generating a nondeterministic predictive model that is probabilistic correct, and compatible with the common motion planning techniques, and designing a motion planning module that plans safe motion for the controlled agent by leveraging the trajectory predictions from the predictive model. To summarize, our main contributions are:

\begin{itemize}
    \item A learning framework that generates future trajectory predictions with probabilistic guarantees of correctness via three post-processing methods
    \item A model predictive controller that leverages the reactivity of uncontrolled agents to generate safe trajectories
    \item Demonstration of the proposed methodology with real highway driving data and simulation of the reactive MPC controller.
\end{itemize}

The paper is structured as follows. First, we present the setup for the predictive model learning, including the generation of the trajectory basis and the affordance approach for scenario description in Section \ref{sec:setup}. Section \ref{sec:NN} presents the neural network structure and proves the generalization bound of the predictive model via several post-processing methods. Section \ref{sec:MPC} presents the model predictive controller leveraging the trajectory prediction, and finally simulation result is shown in Section \ref{sec:sim}.

\section{Setup of the predictive model learning}\label{sec:setup}
\subsection{Trajectory prediction with basis}
One main difficulty of trajectory prediction is the high dimension of the output. Suppose the state $x\in\mathbb{R}^{n_x}$, then a trajectory $\mathbf{x}$ containing $T$ time steps belongs to $\underbrace{\mathbb{R}^{n_x}\times\dots\mathbb{R}^{n_x}}_\text{T}$, whose dimension is $n_xT$. The output dimension quickly becomes too big as $T$ increases and learning such a predictive model becomes intractable.

Inspired by \cite{phan2019covernet}, a trajectory basis method is used to reduce the output dimension. The idea is to use a trajectory basis consisting of a finite collection of base trajectories to represent all possible trajectories with a fixed horizon up to a certain difference threshold. The output of the predictive model is then a binary vector, indicating whether each of the base trajectories is possible under the scenario. Different from \cite{phan2019covernet}, our goal is not to find the most probable trajectory, but instead to find a set containing all possible trajectories in a given scenario, and to provide a probabilistic guarantee on the learned predictive model in terms of covering all possible behavior of the human-driven vehicle, assuming that the training data is a reasonable sample of such behavior.

We let $\mathbf{x}$ denote a state trajectory where $\mathbf{x}(t)\in\mathbb{R}^{n_x}$. Given a training set $\Omega$ consisting of state trajectories lasting for a fixed horizon $T$, we would like to cover all trajectories observed with a finite set of trajectories $\mathcal{B}:=\{\mathbf{x}_i\}_{i=1}^M$, where $M$ is the cardinality. $\mathcal{B}$ is also called the trajectory basis, and satisfies
\begin{equation*}
    \forall \mathbf{x}\in \Omega,~\exists \mathbf{x}_i\in\mathcal{B}~\mathrm{s.t.}~\left\|\mathbf{x}-\mathbf{x}_i\right\|_\mathcal{A} \le \epsilon, 
\end{equation*}
where $||\cdot||_\mathcal{A}$ denotes the atomic norm with atom set $\mathcal{A}$, see \cite{chandrasekaran2012convex} for more details. 
As an example, Fig. \ref{fig:traj_base} shows one base trajectory with the uncertainty region at each sampling time, and another trajectory shown in a dashed line that is close enough to the trajectory base such that the position at each sampling time falls into the uncertainty region. The size of the uncertainty region is determined by the atom set $\mathcal{A}$. With $\mathcal{A}$ fixed, for any two trajectories $\mathbf{x}$ and $\mathbf{x}'$, we can compute their distance in terms of the atomic norm. For a given training set $\Omega$ and a given threshold $\epsilon$, we can compute a graph where each trajectory within the set is a node, and an edge exists between two nodes if and only if their distance is below $\epsilon$. Next, we want to find a trajectory basis $\mathcal{B}$ with a minimum cardinality that covers all trajectories in $\Omega$, i.e., for each trajectory segment in $\Omega$, there exists a base trajectory in $\mathcal{B}$ that is $\epsilon$-close. We say such a trajectory basis $\mathcal{B}$ is an $\epsilon$-covering of the training set $\Omega$, and the process of identifying such a trajectory basis is called sparsification. This is a well-studied set covering problem. Although solving for the exact solution is NP-hard, a greedy algorithm is shown to perform well enough due to the submodularity of the problem, c.f. \cite{fujito2000approximation}. For the example application of this paper, the highway driving problem, we extracted more than 140000 trajectory segments that lasted for 3 seconds on US-101 highway from the NGSIM project \cite{alexiadis2004next}, and the trajectory basis contains 17 base trajectories after sparsification with a greedy algorithm.

\begin{rem}
To make the parameterization less sensitive towards longitudinal speed, each of the base trajectories is represented as the position difference between the actual trajectory and the imaginary trajectory if the vehicle were to follow the current speed forward. In this way, the base trajectory represents the deviation from normal driving, and thus, we do not need multiple basis trajectories for similar driving behavior under different longitudinal speed.
\end{rem}

\begin{figure}
    \centering
    \includegraphics{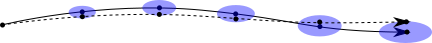}
    \caption{Example base trajectory and uncertainty region, where the solid trajectory is a base trajectory $\mathbf{x}_i\in\mathcal{B}$, and the dashed line is a trajectory in $\Omega$, the blue ellipses are the uncertainty region determined by the atom set $\mathcal{A}$.}
    \label{fig:traj_base}
\end{figure}

\begin{figure}[ht]
\hspace{-0.3cm}\begin{minipage}[b]{0.35\textwidth}
\centering
  \includegraphics[scale =0.37]{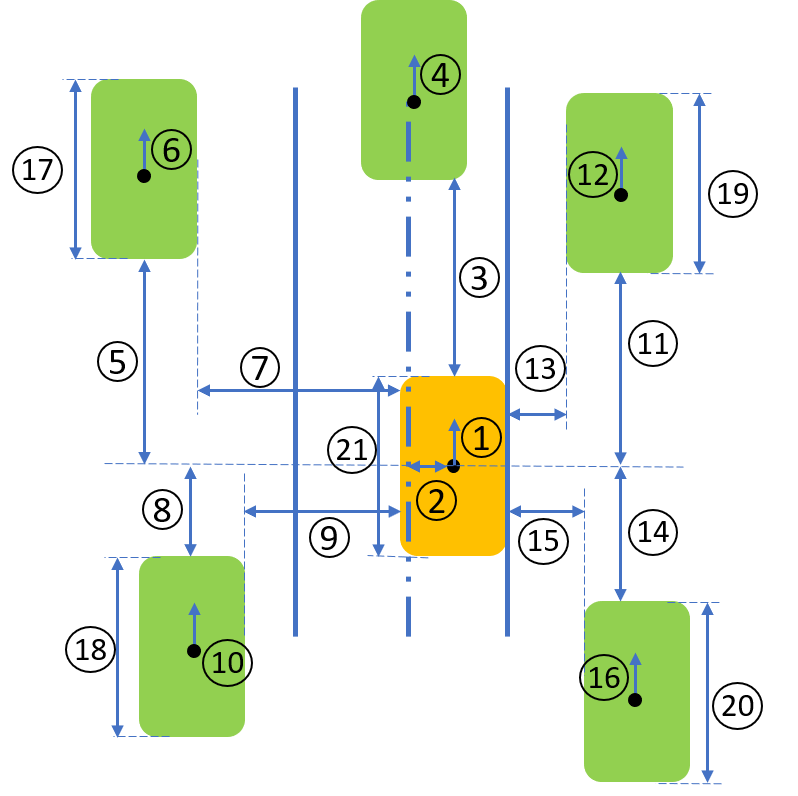}
\end{minipage}
\noindent\begin{minipage}[t]{0.6\textwidth}
\vspace{-4.5cm}
 \centering
  \begin{tabular}{|c|c|c|c|}
1&Forward velocity                &12 & Right front velocity   \\
2&Distance to lane center          &13  & Right front X clearance            \\
3&Forward clearance                &14 & Right rear Y clearance     \\
4&Forward vehicle velocity         &15 & Right rear X clearance     \\
5&Left front Y clearance           &16      & Right rear velocity       \\
6&Left front velocity              &17        & Left front vehicle length   \\
7&Left front X clearance      &18 & Left rear vehicle length\\
8 & Left rear Y clearance     &19& Right front vehicle length\\
9 & Left rear X clearance    &20 & Right rear vehicle length\\
10 & Left rear velocity     &21& Ego-vehicle length\\
11 & Right front Y clearance & & \\
\end{tabular}
\end{minipage}
\caption{Definition of the affordance features with graphic illustrations}
    \label{fig:affordance}
    \vspace{-0.5cm}
\end{figure}

\subsection{Affordance}\label{sec:affordance}
The input to the predictive model describes the scenario, including information about the ego-vehicle (the vehicle whose future motion is to be predicted) and the adjacent vehicles. In \cite{chen2019counter}, the lane-change scenario was described by a single state vector since only one nearby vehicle was considered. However, when multiple adjacent vehicles are involved, traditional multiagent state representation (stacking states of all agents) is not applicable because (i) the number of agents in a scenario is not fixed (ii) the representation is not permutation invariant, i.e., the total state vector changes as the ordering of agents changes, which is problematic for both training and post-processing. One prominent example of a permutation invariant description is graphic representations (and similarly occupancy grid), which was used in \cite{phan2019covernet}. Other examples include the occupancy network \cite{mescheder2019occupancy} and the Pointnet \cite{qi2017pointnet}, both employing special neural network structures.

Since highway driving is considered in this paper, which happens in a highly structured environment, we propose the affordance approach inspired by \cite{chen2015deepdriving}, which is easy to implement and analyze. Affordance is simply a collection of features that describes the driving scenario on the highway, as shown in Fig. \ref{fig:affordance}. It is a comprehensive representation of the driving scenario of the ego-vehicle, including the position of surrounding vehicles and the environment, and it is applicable to an arbitrary number of adjacent vehicles since when a certain vehicle is missing in one position, the corresponding affordance feature simply take the default value, showing that there is no vehicle in that position. Obviously, the affordance representation is permutation invariant by nature.

\subsection{Collision check and output label}\label{sec:collision_check}

Given a scenario described by the affordance vector, we can perform a collision check assuming all vehicles maintain their current speeds, which is implemented with simple algebraic calculations. For each data point in the training set consisting of the affordance vector and the trajectory segment, there are three possible output flags. Firstly, the base trajectory with the minimum distance to the trajectory segment of the data point is marked as positive, with the output flag equal to 1. Note that the base trajectory with minimum distance is always $\epsilon$-close to the data point since the trajectory basis is an $\epsilon$-covering of the training set. For the rest of the base trajectories, we perform collision checks and separate them into two groups, ones that cause collision within the finite horizon and ones that do not. Although the output flag for both groups is zero, later in training, they are given different weights, i.e., the false positive corresponding to base trajectories that lead to collisions are more heavily penalized. To differentiate the two groups, we say the output flag is $0$ if the base trajectory does not lead to a collision, and $\bar{0}$ if it does. The output then is a vector with the $i$th entry equal to one and all other entries equal to zero, where $i$ is the index of the corresponding base trajectory.

\section{Neural network classification with generalization bound}\label{sec:NN}

\subsection{Neural network training}\label{sec:nn_training}
With the input and output of the predictive model specified, a classifier $f:\mathbb{R}^m\to \mathbb{R}^M$ that outputs a binary vector is trained, where $m$ is the dimension of the affordance vector and $M$ is the cardinality of the trajectory basis. We chose a neural network for its strong expressibility, and the neural network is a simple feedforward network with two hidden layers using ReLU as the activation function followed by a sigmoid layer so that the output is a vector with entries between 0 and 1. Let $x$, $y$ be the input and output of the neural network, respectively, and $z$ be the ground truth of the output. As discussed in Section \ref{sec:collision_check}, the entries of $z$ can be 1, 0, or $\bar{0}$. The loss function is then defined as
\begin{equation*}
    J(y,z) = \sum\limits_{i=1}^N w_1\mathds{1}_{z_i=1}\mathrm{ReLU}(\gamma_1-y_i)+w_0\mathds{1}_{z_i=0}\mathrm{ ReLU}(y_i-\gamma_2)+w_{\bar{0}}\mathds{1}_{z_i=\bar{0}}\mathrm{ReLU}(y_i-\gamma_2),
\end{equation*}
where $\gamma_1,\gamma_2\in[0,1]$ are constants chosen to add robustness to the classifier, $w_1,w_0,w_{\bar{0}}$ are the weights for the three classes of outputs and $\mathrm{ReLU}(x)=\max(0,x)$ is the standard ReLU function. $w_{\bar{0}}$ is chosen to be higher than $w_0$ so that false positive predictions corresponding to trajectories leading to collisions are penalized more heavily.

\subsection{Generalization bound}
Since the application is safety-critical, we would like to obtain generalization bound for the predictive model. Generalization bounds based on VC-dimensions~\cite{shalev2014understanding} cannot be applied directly here since computing VC-dimension for Neural Networks is a challenging problem and the bounds for such VC-dimensions are usually very coarse~\cite{sontag1998vc,bartlett2003vapnik}. Moreover, we care more about the upper bound on the false-negative rate than the false-positive since the former leads to guarantee on safety while the latter simply means conservatism.
Here we provide three alternative methods to obtain the generalization bound for a neural network, post-bloating, support vector machine retraining, and conformal regression. We start by reviewing the theory of random convex program.

Let $P[K]$ denote a (minimization) optimization problem with a known objective function and constraint set $K$, and let $\text{Obj}[K]$ denote the optimal objective value of $P[K]$. A constraint $k$ is a supporting constraint if $\text{Obj}[K\backslash \set{k}]<\text{Obj}[K]$.
The setup for an RCP is the following:

\begin{equation}\label{eq:RCP}
  \begin{aligned}
\min \;&J(\alpha)\;\\
\rm{s.t.}~& \alpha\in Q(\delta_i), \forall {\delta_1},...,{\delta_N}\;\textrm{i.i.d samples of } \delta,
\end{aligned}
\end{equation}
where $\alpha\in\mathbb{R}^n$, $Q(\delta_i)\subseteq \mathbb{R}^n$ is a convex set determined by $\delta_i$, and $J(\alpha)$ is convex. $\delta\in\Delta$ is a random variable in the space $\Delta$ and $\left\{\delta_i\right\}$ are independently identically distributed samples of $\delta$. Each $\delta_i$ would pose a convex constraint on $\alpha$. If we randomly draw $N$ samples of $\delta$, and denote it as $\omega\doteq\delta_{1:N}\in\Delta^N$, then let $Q(\omega)\doteq\bigcap\nolimits_{i=1}^{N}{Q(\delta_i)}$, define
\begin{equation}\label{eq:V_def}
  V^*(\omega) = \mathbb{P}\left\{\delta\in\Delta:\text{Obj}([Q(\omega),Q(\delta)])>\text{Obj}[Q(\omega)]\right\},
\end{equation}
 which is the probability that an additional sample added on top of $\omega$ would change the objective value of the original optimization with constraints determined by $\omega$. \cite{calafiore2010random} gives upper bound on $\mathbb{P}(V^*(\omega)\ge\epsilon)$ given $1\ge\epsilon>0$ for a randomly drawn sequence of samples $\omega$. First, we recall the following relevant lemma from \cite{calafiore2010random}:
 \begin{lem}\label{lemma:RCP}
  Consider the random convex program in \eqref{eq:RCP} where $\alpha\in\mathbb{R}^n$. When $N\ge \zeta$, $
  \mathbb{P}\left\{\omega\in\Delta^N:V^*(\omega)>\epsilon\right\}\le\Phi(\epsilon,\zeta-1,N)\le\Phi(\epsilon,n,N),
$
where $\zeta$ is the Helly's dimension denoting the maximum number of supporting constraints, which is bounded by $n+1$, and $
    \Phi(\epsilon,k,N)=\sum\nolimits_{j=0}^{k}\binom{N}{j}\epsilon^j(1-\epsilon)^{N-j}$ is the cumulative distribution of a binomial random variable, that is, the probability of getting no more than $k$ successes in $N$ Bernoulli experiments with success rate $\epsilon$.
 \end{lem}
 This is Theorem 3.3 in \cite{calafiore2010random}, which shows that the result of the RCP is likely to be true for unseen $\delta$ drawn from the same distribution under large $N$ and small $n$.

Applying the RCP theory on our problem, the post-bloating procedure can be described as the following procedure:


\begin{itemize}
    \item Split the training set into two sets $\Omega_1$, $\Omega_2$ with $N_1$ and $N_2$ data points
    \item Use $\Omega_1$ to train a neural network as described in Section \ref{sec:nn_training}
    \item Post-bloat $f$ using $\Omega_2$ with the following optimization
\begin{equation}
\begin{aligned}
    c^\star=\mathop{\min}\limits_{c\in\mathbb{R}^M} & \sum c_i \\
    \mathrm{s.t.}~&\forall (x,z)\in\Omega_2, (z_i=1)\Rightarrow [f(x)]_i>=c_i,
\end{aligned}
\end{equation} 
which is obtained by simply taking the minimum over the output of $f$ with input being $\Omega_2$.
\item Use $c^\star$ as the threshold of $f$, i.e., $[f(x)]_i\ge c_i^\star$ implies the $i$th basis trajectory is possible, otherwise impossible.
\end{itemize}

Then we have the following theorem:

\begin{thm}\label{thm:RCP}
Suppose $\Omega_2$ is a data set consisting of i.i.d. sampled data points, let $f,c^\star$ be the classifier obtained with the post-bloating procedure, then for an unseen data point $(x,z)$ from the same distribution,
\begin{equation*}
     \mathbb{P}\left\{    \mathbb{P}\left\{\exists z_i=1\wedge [f(x)]_i<c_i^\star\right\}>\epsilon       \right\}\le\Phi(\epsilon,M+1,N_2)
\end{equation*}
\end{thm}
\begin{proof}
The post bloating process is a random convex program with $N_2$ constraints drawn i.i.d. from a distribution, the Helly's dimension is upper bounded by $M+1$ since the decision variable $c\in\mathbb{R}^M$ by Lemma 2.3 of \cite{calafiore2010random}. Therefore, the conclusion follows from Lemma \ref{lemma:RCP}.
\end{proof}

The SVM approach works similarly, and we put a detailed description in the supplementary material. 
Conformal regression employs a different theory, yet the procedure is very similar to post-bloating as it also splits the training set, uses one to train the neural network, and uses the other to adjust the threshold. We put the theory and procedure of conformal regression in the supplementary material. In practice, we found the post-bloating giving similar performance to the SVM and conformal regression approach, and yet the former method is much easier to implement and has a smaller Helly's dimension, which leads to a stronger probabilistic guarantee on the correctness of the predictive model. Thus we only include the result using the post-bloating method in this paper.

For the highway autonomous driving example, we use a separate data set to calculate the empirical probability of false negative predictions after post-bloating with data sets with different sizes. $\epsilon_{99\%}$ denotes the upper bound on the false negative rate with confidence $99\%$, i.e., with probability $99\%$, the false negative rate is below $\epsilon_{99\%}$. As shown in Table \ref{tab:RCP}, the false negative rate decreases as  all the empirical false negative rate decreases as the number of data points used for post-bloating increases, and the empirical rate is always upper bounded by the RCP bound $\epsilon_{99\%}$, since $99\%$ is a pretty high confidence level. 

\begin{table}[]
\caption{False negative rate after post-bloating with data sets with various sizes}
\centering
\begin{tabular}{c|c|c}

Size of post-bloating data set & $\epsilon_{99\%}$ by RCP & \multicolumn{1}{l}{Empirical false negative rate} \\ \hline
15946                         & 0.19\%                   & 0.093\%                                           \\
23919                         & 0.103\%                  & 0.05\%                                            \\
31893                         & 0.077\%                  & 0.05\%                                            \\
39866                         & 0.062\%                  & 0.035\%                                           \\
47839                         & 0.052\%                  & 0.032\%                                           \\
55813                         & 0.044\%                  & 0.014\%                                           \\
63786                         & 0.039\%                  & 0.016\%                                     
\end{tabular}

\label{tab:RCP}
\vspace{-0.5cm}
\end{table}



Of course, we are limited by the amount of data available and the rate of false negative is not small enough for realistic self-driving applications yet (human drivers typically experience 1 death in 100 million miles of driving). However, there is an important distinction between prediction error and fatal crash since An unpredicted maneuver by an uncontrolled vehicle does not necessarily lead to an accident, let alone a fatal one. The relationship between the error of the predictive model and an actual crash caused by the prediction error needs further research.

\section{Model predictive control with a predictive model}\label{sec:MPC}
In this section we present a motion planning algorithm based on Model Predictive Control (MPC) that works with the predictive model. We assume that the perception module and the predictive model give us the current and predicted positions of the $N_{\textrm{obst}}$ uncontrolled vehicles $(X^{i}_k, Y^{i}_k)$ for $i\in\{1, \ldots, N_{\textrm{obst}}\}$ and $k \in \{0, \ldots, N\}$.  
Then for each vehicle, we compute an ellipse with semi-axes $(a^i, b^i)$, which contains the predicted future position of the $i$th vehicle with the error bound, and we design an MPC strategy to compute the control action. The Euler discretized Dubin's car model is used as the dynamic model:
\begin{equation}
    x_k=\begin{bmatrix}
    X_k &Y_k& v_k &\psi_k
    \end{bmatrix}^\intercal,\quad x_{k+1}=f(x_k,u_k)=x_k +\begin{bmatrix}
    v_k \cos\psi& v_k \sin\psi& a_k&r_k 
    \end{bmatrix}^\intercal dt,
\end{equation}
where $X_k,Y_k,v_k,\psi_k$ are the longitudinal, lateral position, vehicle velocity, and heading angle at time $k$. The input $u_k$ are the acceleration $a_k$ and yaw rate $r_k$.

At each time $t$, we solve the following finite-time optimal control problem:
\begin{equation}\label{eq:FTOCP}
\begin{aligned}
\min_{\substack{ u_0,\ldots, u_{N-1} \\ \gamma_0,\ldots, \gamma_{N-1} }}\quad & \sum_{k=0}^{N-1} \big( h(x_k,u_k) + c_k \gamma_k^2 \big) + Q(x_N) \\
\text{s.t.} \quad \quad & x_{k+1} = f(x_k, u_k) \\
& \Bigg(\frac{X_k - X_k^{i,j}}{a^i+a}\Bigg)^2 + \Bigg(\frac{Y_k - Y_k^{i,j}}{b^i+b}\Bigg)^2 + \gamma_k \geq 1, u_k \in \mathcal{U} \\
& x_0 = x(t)\\
& \forall k \in \{0,\ldots, N\}, \forall j \in \{1, \ldots, N_{\textrm{pred}}^i\}, \forall i \in \{1,\ldots, N_\textrm{obst}\}
\end{aligned}
\end{equation}
where $N_\textrm{pred}^i$ denotes the number of predicted trajectories for the $i$th uncontrolled vehicle and for $j \in \{1, \ldots, N_\textrm{pred}^i\}$ we have that $[(X_0^{i,j}, Y_0^{i,j}), \ldots, (X_N^{i,j}, Y_N^{i,j})]$ represents the possible $j$th trajectory associated with the $i$th uncontrolled vehicle. $(a,b)$ are the semi-axes of the ellipse associated with the controlled vehicle and $(a^i,b^i)$ are the semi-axes of the $i$th uncontrolled vehicle.

The above finite-time optimal control problem computes a trajectory for the nominal model $x_{k+1} = f(x_k, u_k)$, which minimized the running cost $h(\cdot, \cdot)$ while avoiding the other vehicles represented by obstacles and satisfying the input constraints, for further details on the obstacle representation we refer to~\cite{rosolia2016autonomous}. Notice that in the above problem, we used the slack variable $\gamma_k$ that allows us to relax the obstacle constraint, and the weight $c_k$ is a tuning parameter that penalizes the constraint violation. Let $[u_0^*, \ldots, u_{N-1}^*]$ be the optimal solution to the finite-time optimal control problem~\eqref{eq:FTOCP} at time $t$, then we apply to the system $\pi^{\textrm{MPC}}\big( x(t) \big) = u_{t|t}^*.$
The above procedure is repeated at the next time step $t+1$ based on the new measurement $x(t+1)$.

\begin{rem}
Although by convention it is the rear vehicle's responsibility to prevent a rear-end collision, for the sake of safety, the MPC still considers an uncontrolled vehicle directly behind the controlled vehicle. This leads to some frontal or side collision when the controlled vehicle is trapped between uncontrolled vehicles, as shown in Section \ref{sec:sim}. When we chose to let the MPC ignore the uncontrolled vehicle directly behind, we observe no frontal or side collisions in simulation.
\end{rem}
\section{Simulation with MPC}\label{sec:sim}

The proposed approach is tested in a simulation environment where the uncontrolled vehicles operate autonomously in a reactive fashion. Each uncontrolled vehicle updates its trajectory by randomly selecting a trajectory from the trajectory basis that is consistent with the trajectory predictor and passes the collision check. It should be emphasized that the collision check only checks potential collisions with other uncontrolled vehicles based on their selected trajectory; and checks collision with the autonomous vehicle assuming constant velocity for the autonomous vehicle if the autonomous vehicle is directly in front of the uncontrolled vehicle, as discussed in Section \ref{sec:MPC}. This check with the autonomous vehicle is to prevent the uncontrolled vehicles from hitting the autonomous vehicle from behind since typically; it is the rear vehicle's responsibility to prevent a rear collision. In any other case, the collision avoidance task for the autonomous vehicle falls on the MPC controller with the predictive model. The autonomous vehicle then calculates the affordance for each adjacent uncontrolled vehicle, pass the affordance to the predictive model to obtain the possible trajectories, and plan its own motion with the MPC controller introduced in Section \ref{sec:MPC} by leveraging the predicted trajectories. Furthermore, the MPC optimization problem is implemented using CasADi~\cite{andersson2019casadi} for automatic differentiation and IPOPT~\cite{wachter2006implementation} to solve the nonlinear program.

The result shows that the MPC controller, together with the trajectory predictor is able to safely plan the motion for the autonomous vehicle by leveraging the reactive behavior of the uncontrolled vehicle. The animation of the simulation can be found in the supplementary material of this paper.
\begin{figure}
    \centering
    \includegraphics[width=1\textwidth]{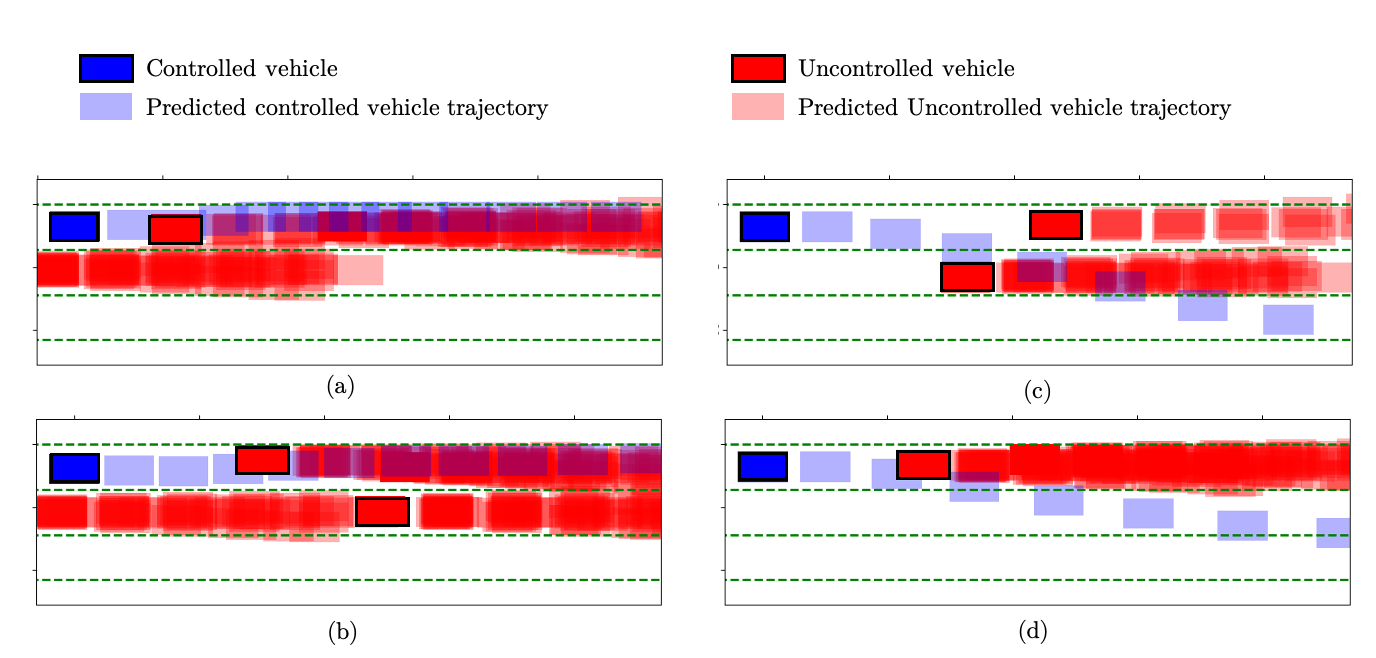}
    \caption{Snapshots of the simulation with multiple uncontrolled vehicles. The blue rectangles represent the autonomous vehicle and the predicted trajectory in light red. The uncontrolled vehicles is depicted in red and the predicted trajectory in light red.}
    \label{fig:sim_snapshot}
\end{figure}

\begin{wrapfigure}{r}{0.65\textwidth}
    \centering
    \includegraphics[scale =0.6]{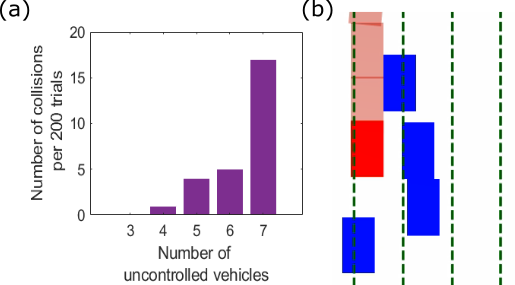}
    \caption{(a) Distribution of collisions (b) a typical collision case}
    \label{fig:stats}
\end{wrapfigure}
Fig. \ref{fig:sim_snapshot} shows the predicted MPC trajectories by the controlled vehicle in four different scenarios. The red shaded patches represent the predicted positions of the uncontrolled agents. These predicted trajectories are used to enforce the time-varying constraints in the MPC problem~\eqref{eq:FTOCP}. As the constraints are time-varying, the controller can predict to cross a location previously occupied by the uncontrolled agent. However, the MPC plans a trajectory such that the controlled and the uncontrolled vehicle do not occupy the same location at the same time. As a result, the controller slows down when is surrounded by uncontrolled vehicles (Fig.~\ref{fig:sim_snapshot} (a) and (b)) and it accelerates to perform an overtaking maneuver when it is safe to do so (Fig.~\ref{fig:sim_snapshot} (c) and (d)).

As a statistical study, we performed 1000 trials with random initial scenarios containing 3 to 7 uncontrolled vehicles (200 trials each), each lasting for 20 seconds. The MPC controller is designed to perform lane changes and overtaking maneuvers, which are challenging tasks for autonomous vehicles. The result of the 1000 trials showed seven frontal or side collisions, and 20 rear-end collisions where the distribution of collisions w.r.t. the number of uncontrolled vehicles is shown in Fig. \ref{fig:stats}(a). However, after inspection, almost all the collision cases were caused by the controlled vehicle surrounded by the uncontrolled vehicles, leaving no safe solutions to the MPC controller. An example snapshot is shown in Fig. \ref{fig:stats}(b). A video describing the sim result can be found in \href{https://youtu.be/E49TH0kPBuo}{https://youtu.be/E49TH0kPBuo}

\begin{rem}
 The simulation environment can be viewed as an over-approximation of the realistic highway environment since the trajectories for the uncontrolled vehicles are randomly picked from the trajectories deemed possible by the predictive model, which over-approximates the set of possible trajectories. This leads to some unrealistic scenarios where multiple uncontrolled vehicles trap the controlled vehicle and collision is inevitable. In future works, we plan to perform extensive testing of the proposed strategy on a more realistic simulation environment emulating the true nature of human-driven vehicles and perform a more thorough statistical study.
\end{rem}

\section{Conclusion}
We propose a predictive modeling framework that predicts possible trajectories within a short horizon for uncontrolled agents based on the scenario in an environment. By applying tools in statistics, we are able to prove a probabilistic guarantee on the correctness of the predictive model after post-processing. The predictive model is compatible with most of the existing motion planning methods, which are typically horizon-based, to plan safe motion for the controlled agent by leveraging the trajectory prediction. We apply the framework to the highway autonomous driving problem where the predictive model takes the affordance, a permutation-invariant scenario representation as the input, and outputs a vector indicating possible trajectories of an uncontrolled vehicle within a finite horizon. We then design an MPC controller that leverages the trajectory prediction to plan safe motion. A simulation study shows promising results that the proposed framework is able to safely navigate the autonomous vehicle in an environment packed with uncontrolled vehicles and perform challenging tasks such as lane change and overtaking.
\renewcommand{\baselinestretch}{0.94}
\newpage
\section{Appendix}
\subsection{SVM for post processing}
SVM was used in \cite{chen2019counter} for the training of the classifier to get generalization bound with the RCP theory, yet we found that it performed poorly with the affordance as input. The main reason is the absence of good nonlinear features. On the other hand, the neural network, as an unsupervised method for classification, is able to automatically find good features for the classification task. However, due to the complicated structure and a large number of parameters, obtaining nice generalization bound has been difficult for neural networks. We combine the two classification methods to get the advantages of both methods. To be specific, the following procedure is proposed, which uses a neural network to find the features for SVM:
\begin{itemize}
    \item Split the training set into two sets $\Omega_1$, $\Omega_2$ with $N_1$ and $N_2$ data points
    \item Use $\Omega_1$ to train a neural network as described in Section \ref{sec:nn_training} and save the network structure up to the last hidden layer, which has $\bar{n}$ neurons, denote the nonlinear function from input to the last hidden layer as $g$.
    \item Let $g$ be the nonlinear feature for the SVM and use $\Omega_2$ to train an SVM following the procedure detailed in Section 3.1 of \cite{chen2019counter}.
\end{itemize}
The generalization bound then follows in a similar fashion as Theorem \ref{thm:RCP} with the number of samples being $N_2$ and Helly's dimension being $\bar{n}+1$.

\subsection{Conformal regression for probabilistic guarantees}
\input{conformal}
\bibliography{mybib}

\end{document}

%% file: conformal.tex

In this section, we discuss a third approach to provide probabilistic guarantees for our trained classifier $f$.
Conformal regression~\cite{lei2018distribution,lei2014distribution} provides a framework to quantify the accuracy of the predictive inference in regression using conformal inference~\cite{vovk2005algorithmic}. We choose to use conformal regression since it can provide valid coverage in finite samples, without making assumptions on the distribution.

Consider i.i.d. regression data $Z_1,\cdots,Z_N$ drawn from an arbitrary joint $\distribution_{XY}$, where each $Z_i = (X_i,Y_i)$ is a random variable in $\reals^n \times \reals$, comprised of a $n$-dimensional feature vectors $X_i$ and a response variable $Y_i$. 
Conformal regression problem is to predict a new response $Y_{N+1}$ from
a new feature value $X_{N+1}$, with no assumptions on $\distribution_{XY}$.
Formally, given a positive value $\varepsilon \in (0,1)$, conformal regression techniques can construct a prediction band $C \subseteq \reals^n \times \reals$ based
on $Z_1,\cdots,Z_n$ with the property that
\begin{equation}
\label{eq:finite-sample-guarantee}
    \prob(Y_{N+1} \in C(X_{N+1})) \geq 1-\varepsilon,
\end{equation}
where the probability is taken over the $N+1$ i.i.d. draws $Z_1,\cdots,Z_N,Z_{N+1} \sim \distribution_{XY}$, and for
a point $x \in \reals^n$ we denote $C(x) = \{y \in \reals: (x, y)\in C\}$.
Such $\varepsilon$ is called the {\em miscoverage level}, and $1-\varepsilon$ is the probability threshold.

Let 
\[
\mu(x) = \mathbb{E}(Y \ | \ X = x), x \in \reals^n
\]
denote the regression function. 
The regression problem is to estimate such conditional mean of the test response $Y_{N+1}$ given the test feature $X_{N+1} = x$.  
Common regression methods use a regression model $g(x,\regpara)$ and minimize the sum of squared residuals of such model on the $N$ training regression data $Z_1,\cdots,Z_N$, where $\regpara$ are the parameters of the regression model.
The estimator for $\mu$ is given by
\[\hat{\mu}(x) = g(x,\hat{\regpara}),\]
where
$\hat{\regpara} = \mbox{argmin}_{\regpara} \frac{1}{N} \sum_{i=1}^{N} (Y_{i} - g(X_i,\regpara))^2 + \R(\regpara)$
and $\R(\regpara)$ is a potential regularizer. Common regression model
$g(x,\regpara)$ includes linear, polynomial, and neural
networks~\cite{friedman2001elements,barron1993universal}.

Let us assume that our classifier $f$ only has a $1$-dimensional output (i.e. $M=1$) for the rest of the section. When $M>1$, the analysis can be done per dimension of the output.
We use a split conformal methods from~\cite{lei2018distribution} to construct prediction intervals that satisfy the finite-sample coverage guarantees as in Equation~\eqref{eq:finite-sample-guarantee}. We begin by splitting the training data into two equal-sized disjoint subsets with indices $\I{1},\I{2}$. Then our classifier $f$ can be seen as the estimator $\hat{\mu}$ is fit to the training set $\{(X_i, Y_i)\}: i \in \I_1\}$.   
Then the algorithm compute the absolute residuals $R_i = |Y_i - f(X_i)|$ on the calibration set $\{(X_i, Y_i)\}: i \in \I_2)$. For a given miscoverage level $\varepsilon$, the algorithm rank $\{R_i: i\in \I_2\}$ and take the $\lceil (N/2+1)(1-\varepsilon) \rceil$th one as the confidence range $d$. Finally, it can be proved that the prediction interval at a new point $X_{N+1}$ is given by such $f$ and $d$ as in Lemma~\ref{thm:conformal}.

\begin{lem}[Theorem 2.2 in~\cite{lei2018distribution}]
\label{thm:conformal}
If $(X_i,Y_i),i=1,\cdots,N$ are i.i.d., $f$ is classifier trained on the training set $\I_{1}$ and $d$ is the $\lceil (N/2+1)(1-\varepsilon) \rceil$th smallest value of the residuals $\{R_i = |Y_i - \hat{\mu}(X_i)|: i\in \I_2\}$,  then for an new i.i.d. draw $(X_{N+1}, Y_{N+1})$, 
\[
\prob(Y_{N+1} \in [f(X_{N+1}) -d, f(X_{N+1}) +d ] ) \geq 1 - \varepsilon.
\]
Moreover, if we assume additionally that the residuals $\{R_i: i\in \I_2\}$ have a continuous joint distribution, then
\[
\prob(Y_{N+1} \in [f(X_{N+1}) -d, f(X_{N+1}) +d ] ) \leq 1 - \varepsilon +\frac{2}{N+2}.
\]
\end{lem}

Generally speaking, as we improve our classifier $f$ of the underlying regression function $\mu$ using more training samples, the resulting conformal prediction interval decreases in length. Intuitively,
this happens because a more accurate $f$ leads to smaller residuals, and conformal intervals are essentially defined by the quantiles of the (augmented) residual distribution. Theoretically, we have the following lemma to quantify the residuals as a function of the training set size. 

\begin{lem}[Theorem 2.3 in~\cite{lei2018distribution}]
Under the conditions of Lemma~\ref{thm:conformal}, there is an absolute constant $c >0$ such that for any $\varepsilon>0$,
\[
\prob\left(\left\lvert \frac{2}{N} \sum_{i \in \I_{2}} \mathds{1}\{Y_i \in [f(X_{i}) -d, f(X_{i}) +d ]\} - (1-\varepsilon) \right\rvert > \varepsilon\right) \leq 2 e^{\left(-cn^2(\varepsilon-\frac{4}{N})^2\right)}.
\]
\end{lem}

When using the conformal regression method, instead of using the output of the classifier $f$ directly, we use the confidence interval $[f(x) -d, f(x) +d ]$ for each base trajectory with the affordance $x$. If $1 \in [f(x) -d, f(x) +d ]$ we mark the trajectory as positive.
Note that Lemma~\ref{thm:conformal} assert marginal coverage guarantees, which should be distinguished with the conditional coverage guarantee $\prob(Y_{N+1} \in C(x) \ | \ X_{N+1} = x) \geq 1-\varepsilon$ for all $x \in \reals^n$. The latter one is a  a much stronger property and hard to be achieved without assumptions on $\distribution_{XY}$.

\begin{table}[]
\centering
\caption{Theoretical and empirical false negative rate using conformal regression analysis}
\begin{tabular}{c|c|c}
Average $d$  & Theoretical $\varepsilon$ & Empirical $\varepsilon$ \\ \hline
0.0972 & 10.0000\% & 10.1639\%
\\
 0.1491 & 1.0000\% & 1.0691\%                                            \\
 0.1681 & 0.5000\%  & 0.5352\%                                       \\
0.1838 & 0.3000\% & 0.3266\%\\
0.1950 & 0.2000\% & 0.2287 \% \\
0.2172 & 0.1000\% & 0.1269\% \\
0.2442 & 0.0400\%  & 0.0603 \%
\end{tabular}
\label{tab:conformal}
\end{table}

Following the idea of conformal regression, we use 23919 trajectories from the NGSIM dataset to calibrate the classifier $f$ with different miscoverage rate $\varepsilon$ and another independent set for validation. The results are reported in Table~\ref{tab:conformal}. The weighted average confidence range $d$ is computed as $\sum_{i=1}^{17} p_i d_i$, where $p_i$ is the percentage of the trajectories in the calibration set corresponding to the $i$th base trajectory and $d_i$ is the confidence range of the $i$th base trajectory. The empirical $\varepsilon$ is computed the portion of trajectories in the testing set whose label is not contained in the prediction interval $[f(x) -d, f(x) +d ]$. 

From Table~\ref{tab:conformal}, we can see that the weighted average confidence range $d$ across the base trajectories increases as the miscoverage rate $\varepsilon$ decreases. The empirical $\varepsilon$ is in general slightly larger then the theoretical $\varepsilon$ which was used to compute $d$ but the difference is very small. We observed that the results did not vary too much when the size of the testing set changed. 

The procedure of using conformal regression to provide probabilistic guarantee is very different from the post-bloating method discussed in Section 3.2. However, comparing with Table 1, we can see that conformal regression can provide a similar level of high assurance on the prediction result while using a fixed set of training/calibration trajectories. We also observed that although the weighted average $d$ across all base trajectories was relatively small, the actual $d$ for each base trajectories varied a lot. For example, with $\varepsilon = 1\%$, the confidence $d$s for the $17$ base trajectories vary from $0.147$ to $0.995$. The large $d$s are mainly due to the lack of data (i.e. only very small portion of the training and calibration data are associated with the corresponding base trajectory).

Comparing to post-bloating, conformal regression can provide more flexibility in terms of choosing the coverage level. However, the performance of the conformal regression-based method heavily depends on the composition of the dataset, and achieves the best result when there are sufficient number of data points for each category. While the post-bloating method is not sensitive about the composition of the dataset as it solves a single optimization that naturally factors in the distribution over categories. We plan to explore in the future how to deal with datasets with an uneven composition over categories.

%% file: main.bbl
\begin{thebibliography}{42}
\providecommand{\natexlab}[1]{#1}
\providecommand{\url}[1]{\texttt{#1}}
\expandafter\ifx\csname urlstyle\endcsname\relax
  \providecommand{\doi}[1]{doi: #1}\else
  \providecommand{\doi}{doi: \begingroup \urlstyle{rm}\Url}\fi

\bibitem[Kuwata et~al.(2008)Kuwata, Fiore, Teo, Frazzoli, and
  How]{kuwata2008motion}
Y.~Kuwata, G.~A. Fiore, J.~Teo, E.~Frazzoli, and J.~P. How.
\newblock Motion planning for urban driving using rrt.
\newblock In \emph{2008 IEEE/RSJ International Conference on Intelligent Robots
  and Systems}, pages 1681--1686. IEEE, 2008.

\bibitem[Chen et~al.(2017)Chen, Peng, and Grizzle]{chen2017fast}
Y.~Chen, H.~Peng, and J.~W. Grizzle.
\newblock Fast trajectory planning and robust trajectory tracking for
  pedestrian avoidance.
\newblock \emph{Ieee Access}, 5:\penalty0 9304--9317, 2017.

\bibitem[Jha et~al.(2017)Jha, Tiwari, Seshia, Sahai, and Shankar]{jha2017telex}
S.~Jha, A.~Tiwari, S.~A. Seshia, T.~Sahai, and N.~Shankar.
\newblock Telex: Passive stl learning using only positive examples.
\newblock In \emph{International Conference on Runtime Verification}, pages
  208--224. Springer, 2017.

\bibitem[Vazquez-Chanlatte et~al.(2018)Vazquez-Chanlatte, Jha, Tiwari, Ho, and
  Seshia]{vazquez2018learning}
M.~Vazquez-Chanlatte, S.~Jha, A.~Tiwari, M.~K. Ho, and S.~Seshia.
\newblock Learning task specifications from demonstrations.
\newblock In \emph{Advances in Neural Information Processing Systems}, pages
  5367--5377, 2018.

\bibitem[{Jin} et~al.(2015){Jin}, {Donzé}, {Deshmukh}, and {Seshia}]{JinX}
X.~{Jin}, A.~{Donzé}, J.~V. {Deshmukh}, and S.~A. {Seshia}.
\newblock Mining requirements from closed-loop control models.
\newblock \emph{IEEE Transactions on Computer-Aided Design of Integrated
  Circuits and Systems}, 34\penalty0 (11):\penalty0 1704--1717, Nov 2015.
\newblock ISSN 0278-0070.
\newblock \doi{10.1109/TCAD.2015.2421907}.

\bibitem[Trautman and Krause(2010)]{trautman2010unfreezing}
P.~Trautman and A.~Krause.
\newblock Unfreezing the robot: Navigation in dense, interacting crowds.
\newblock In \emph{2010 IEEE/RSJ International Conference on Intelligent Robots
  and Systems}, pages 797--803. IEEE, 2010.

\bibitem[Aoude et~al.(2013)Aoude, Luders, Joseph, Roy, and
  How]{aoude2013probabilistically}
G.~S. Aoude, B.~D. Luders, J.~M. Joseph, N.~Roy, and J.~P. How.
\newblock Probabilistically safe motion planning to avoid dynamic obstacles
  with uncertain motion patterns.
\newblock \emph{Autonomous Robots}, 35\penalty0 (1):\penalty0 51--76, 2013.

\bibitem[Ziebart et~al.(2010)Ziebart, Bagnell, and Dey]{ziebart2010modeling}
B.~D. Ziebart, J.~A. Bagnell, and A.~K. Dey.
\newblock Modeling interaction via the principle of maximum causal entropy.
\newblock 2010.

\bibitem[Sadigh et~al.(2016)Sadigh, Sastry, Seshia, and
  Dragan]{sadigh2016planning}
D.~Sadigh, S.~S. Sastry, S.~A. Seshia, and A.~D. Dragan.
\newblock Planning for autonomous cars that leverage effects on human actions.
\newblock In \emph{Proceedings of Robotics: Science and Systems ({RSS})}, June
  2016.
\newblock \doi{10.15607/RSS.2016.XII.029}.

\bibitem[{Kuderer} et~al.(2015){Kuderer}, {Gulati}, and {Burgard}]{Kuderer2015}
M.~{Kuderer}, S.~{Gulati}, and W.~{Burgard}.
\newblock Learning driving styles for autonomous vehicles from demonstration.
\newblock In \emph{2015 IEEE International Conference on Robotics and
  Automation (ICRA)}, pages 2641--2646, May 2015.
\newblock \doi{10.1109/ICRA.2015.7139555}.

\bibitem[Bhattacharyya et~al.(2018)Bhattacharyya, Phillips, Wulfe, Morton,
  Kuefler, and Kochenderfer]{bhattacharyya2018multi}
R.~P. Bhattacharyya, D.~J. Phillips, B.~Wulfe, J.~Morton, A.~Kuefler, and M.~J.
  Kochenderfer.
\newblock Multi-agent imitation learning for driving simulation.
\newblock In \emph{2018 IEEE/RSJ International Conference on Intelligent Robots
  and Systems (IROS)}, pages 1534--1539. IEEE, 2018.

\bibitem[Althoff et~al.(2009)Althoff, Stursberg, and Buss]{althoff2009model}
M.~Althoff, O.~Stursberg, and M.~Buss.
\newblock Model-based probabilistic collision detection in autonomous driving.
\newblock \emph{IEEE Transactions on Intelligent Transportation Systems},
  10\penalty0 (2):\penalty0 299--310, 2009.

\bibitem[Kumagai and Akamatsu(2006)]{kumagai2006prediction}
T.~Kumagai and M.~Akamatsu.
\newblock Prediction of human driving behavior using dynamic bayesian networks.
\newblock \emph{IEICE TRANSACTIONS on Information and Systems}, 89\penalty0
  (2):\penalty0 857--860, 2006.

\bibitem[Bai et~al.(2015)Bai, Cai, Ye, Hsu, and Lee]{bai2015intention}
H.~Bai, S.~Cai, N.~Ye, D.~Hsu, and W.~S. Lee.
\newblock Intention-aware online pomdp planning for autonomous driving in a
  crowd.
\newblock In \emph{2015 ieee international conference on robotics and
  automation (icra)}, pages 454--460. IEEE, 2015.

\bibitem[Chen et~al.(2018)Chen, Sohani, and Peng]{chen2018modelling}
Y.~Chen, N.~Sohani, and H.~Peng.
\newblock Modelling of uncertain reactive human driving behavior: a
  classification approach.
\newblock In \emph{2018 IEEE Conference on Decision and Control (CDC)}, pages
  3615--3621. IEEE, 2018.

\bibitem[Chen et~al.(2019)Chen, Dathathri, Phan-Minh, and
  Murray]{chen2019counter}
Y.~Chen, S.~Dathathri, T.~Phan-Minh, and R.~M. Murray.
\newblock Counter-example guided learning of bounds on environment behavior.
\newblock \emph{In proceedings of 2019 Conference on Robotic Learning (Corl)},
  2019.

\bibitem[Phan-Minh et~al.(2019)Phan-Minh, Grigore, Boulton, Beijbom, and
  Wolff]{phan2019covernet}
T.~Phan-Minh, E.~C. Grigore, F.~A. Boulton, O.~Beijbom, and E.~M. Wolff.
\newblock Covernet: Multimodal behavior prediction using trajectory sets.
\newblock \emph{arXiv preprint arXiv:1911.10298}, 2019.

\bibitem[LaValle et~al.(2001)LaValle, Kuffner, Donald,
  et~al.]{lavalle2001rapidly}
S.~M. LaValle, J.~J. Kuffner, B.~Donald, et~al.
\newblock Rapidly-exploring random trees: Progress and prospects.
\newblock \emph{Algorithmic and computational robotics: new directions},
  \penalty0 (5):\penalty0 293--308, 2001.

\bibitem[Borrelli et~al.(2017)Borrelli, Bemporad, and
  Morari]{borrelli2017predictive}
F.~Borrelli, A.~Bemporad, and M.~Morari.
\newblock \emph{Predictive control for linear and hybrid systems}.
\newblock Cambridge University Press, 2017.

\bibitem[Vasquez et~al.(2004)Vasquez, Large, Fraichard, and
  Laugier]{vasquez2004high}
D.~Vasquez, F.~Large, T.~Fraichard, and C.~Laugier.
\newblock High-speed autonomous navigation with motion prediction for unknown
  moving obstacles.
\newblock In \emph{2004 IEEE/RSJ International Conference on Intelligent Robots
  and Systems (IROS)(IEEE Cat. No. 04CH37566)}, volume~1, pages 82--87. IEEE,
  2004.

\bibitem[Shah et~al.(2016)Shah, {\v{S}}vec, Bertaska, Sinisterra, Klinger, von
  Ellenrieder, Dhanak, and Gupta]{shah2016resolution}
B.~C. Shah, P.~{\v{S}}vec, I.~R. Bertaska, A.~J. Sinisterra, W.~Klinger, K.~von
  Ellenrieder, M.~Dhanak, and S.~K. Gupta.
\newblock Resolution-adaptive risk-aware trajectory planning for surface
  vehicles operating in congested civilian traffic.
\newblock \emph{Autonomous Robots}, 40\penalty0 (7):\penalty0 1139--1163, 2016.

\bibitem[Balkan et~al.(2017)Balkan, Tabuada, Deshmukh, Jin, and
  Kapinski]{Balkan}
A.~Balkan, P.~Tabuada, J.~V. Deshmukh, X.~Jin, and J.~Kapinski.
\newblock Underminer: A framework for automatically identifying nonconverging
  behaviors in black-box system models.
\newblock \emph{ACM Trans. Embed. Comput. Syst.}, 17\penalty0 (1):\penalty0
  20:1--20:28, Dec. 2017.
\newblock ISSN 1539-9087.
\newblock \doi{10.1145/3122787}.
\newblock URL \url{http://doi.acm.org/10.1145/3122787}.

\bibitem[Fan et~al.(2017)Fan, Qi, Mitra, and Viswanathan]{DRYVR}
C.~Fan, B.~Qi, S.~Mitra, and M.~Viswanathan.
\newblock {DRYVR:} data-driven verification and compositional reasoning for
  automotive systems.
\newblock \emph{CoRR}, abs/1702.06902, 2017.
\newblock URL \url{http://arxiv.org/abs/1702.06902}.

\bibitem[Haesaert et~al.(2015)Haesaert, den Hof, and Abate]{sofiedatadriven}
S.~Haesaert, P.~M. J.~V. den Hof, and A.~Abate.
\newblock Data-driven and model-based verification: a bayesian identification
  approach.
\newblock \emph{CoRR}, abs/1509.03347, 2015.
\newblock URL \url{http://arxiv.org/abs/1509.03347}.

\bibitem[Chandrasekaran et~al.(2012)Chandrasekaran, Recht, Parrilo, and
  Willsky]{chandrasekaran2012convex}
V.~Chandrasekaran, B.~Recht, P.~A. Parrilo, and A.~S. Willsky.
\newblock The convex geometry of linear inverse problems.
\newblock \emph{Foundations of Computational mathematics}, 12\penalty0
  (6):\penalty0 805--849, 2012.

\bibitem[Fujito(2000)]{fujito2000approximation}
T.~Fujito.
\newblock Approximation algorithms for submodular set cover with applications.
\newblock \emph{IEICE Transactions on Information and Systems}, 83\penalty0
  (3):\penalty0 480--487, 2000.

\bibitem[Alexiadis et~al.(2004)Alexiadis, Colyar, Halkias, Hranac, and
  McHale]{alexiadis2004next}
V.~Alexiadis, J.~Colyar, J.~Halkias, R.~Hranac, and G.~McHale.
\newblock The next generation simulation program.
\newblock \emph{Institute of Transportation Engineers. ITE Journal},
  74\penalty0 (8):\penalty0 22, 2004.

\bibitem[Mescheder et~al.(2019)Mescheder, Oechsle, Niemeyer, Nowozin, and
  Geiger]{mescheder2019occupancy}
L.~Mescheder, M.~Oechsle, M.~Niemeyer, S.~Nowozin, and A.~Geiger.
\newblock Occupancy networks: Learning 3d reconstruction in function space.
\newblock In \emph{Proceedings of the IEEE Conference on Computer Vision and
  Pattern Recognition}, pages 4460--4470, 2019.

\bibitem[Qi et~al.(2017)Qi, Su, Mo, and Guibas]{qi2017pointnet}
C.~R. Qi, H.~Su, K.~Mo, and L.~J. Guibas.
\newblock Pointnet: Deep learning on point sets for 3d classification and
  segmentation.
\newblock In \emph{Proceedings of the IEEE conference on computer vision and
  pattern recognition}, pages 652--660, 2017.

\bibitem[Chen et~al.(2015)Chen, Seff, Kornhauser, and
  Xiao]{chen2015deepdriving}
C.~Chen, A.~Seff, A.~Kornhauser, and J.~Xiao.
\newblock Deepdriving: Learning affordance for direct perception in autonomous
  driving.
\newblock In \emph{Proceedings of the IEEE International Conference on Computer
  Vision}, pages 2722--2730, 2015.

\bibitem[Shalev-Shwartz and Ben-David(2014)]{shalev2014understanding}
S.~Shalev-Shwartz and S.~Ben-David.
\newblock \emph{Understanding machine learning: From theory to algorithms}.
\newblock Cambridge university press, 2014.

\bibitem[Sontag(1998)]{sontag1998vc}
E.~D. Sontag.
\newblock Vc dimension of neural networks.
\newblock \emph{NATO ASI Series F Computer and Systems Sciences}, 168:\penalty0
  69--96, 1998.

\bibitem[Bartlett and Maass(2003)]{bartlett2003vapnik}
P.~L. Bartlett and W.~Maass.
\newblock Vapnik-chervonenkis dimension of neural nets.
\newblock \emph{The handbook of brain theory and neural networks}, pages
  1188--1192, 2003.

\bibitem[Calafiore(2010)]{calafiore2010random}
G.~C. Calafiore.
\newblock Random convex programs.
\newblock \emph{SIAM Journal on Optimization}, 20\penalty0 (6):\penalty0
  3427--3464, 2010.

\bibitem[Rosolia et~al.(2016)Rosolia, De~Bruyne, and
  Alleyne]{rosolia2016autonomous}
U.~Rosolia, S.~De~Bruyne, and A.~G. Alleyne.
\newblock Autonomous vehicle control: A nonconvex approach for obstacle
  avoidance.
\newblock \emph{IEEE Transactions on Control Systems Technology}, 25\penalty0
  (2):\penalty0 469--484, 2016.

\bibitem[Andersson et~al.(2019)Andersson, Gillis, Horn, Rawlings, and
  Diehl]{andersson2019casadi}
J.~A. Andersson, J.~Gillis, G.~Horn, J.~B. Rawlings, and M.~Diehl.
\newblock Casadi: a software framework for nonlinear optimization and optimal
  control.
\newblock \emph{Mathematical Programming Computation}, 11\penalty0
  (1):\penalty0 1--36, 2019.

\bibitem[W{\"a}chter and Biegler(2006)]{wachter2006implementation}
A.~W{\"a}chter and L.~T. Biegler.
\newblock On the implementation of an interior-point filter line-search
  algorithm for large-scale nonlinear programming.
\newblock \emph{Mathematical programming}, 106\penalty0 (1):\penalty0 25--57,
  2006.

\bibitem[Lei et~al.(2018)Lei, G’Sell, Rinaldo, Tibshirani, and
  Wasserman]{lei2018distribution}
J.~Lei, M.~G’Sell, A.~Rinaldo, R.~J. Tibshirani, and L.~Wasserman.
\newblock Distribution-free predictive inference for regression.
\newblock \emph{Journal of the American Statistical Association}, 113\penalty0
  (523):\penalty0 1094--1111, 2018.

\bibitem[Lei and Wasserman(2014)]{lei2014distribution}
J.~Lei and L.~Wasserman.
\newblock Distribution-free prediction bands for non-parametric regression.
\newblock \emph{Journal of the Royal Statistical Society: Series B (Statistical
  Methodology)}, 76\penalty0 (1):\penalty0 71--96, 2014.

\bibitem[Vovk et~al.(2005)Vovk, Gammerman, and Shafer]{vovk2005algorithmic}
V.~Vovk, A.~Gammerman, and G.~Shafer.
\newblock \emph{Algorithmic learning in a random world}.
\newblock Springer Science \& Business Media, 2005.

\bibitem[Friedman et~al.(2001)Friedman, Hastie, and
  Tibshirani]{friedman2001elements}
J.~Friedman, T.~Hastie, and R.~Tibshirani.
\newblock \emph{The elements of statistical learning}, volume~1.
\newblock Springer series in statistics New York, 2001.

\bibitem[Barron(1993)]{barron1993universal}
A.~R. Barron.
\newblock Universal approximation bounds for superpositions of a sigmoidal
  function.
\newblock \emph{IEEE Transactions on Information theory}, 39\penalty0
  (3):\penalty0 930--945, 1993.

\end{thebibliography}
